\documentclass{article}

\usepackage{colt11e}
\usepackage[round,comma]{natbib}
\usepackage{amsmath,amssymb}
\usepackage{mathrsfs}
\usepackage{mathtools}

\newenvironment{proof*}{\noindent{\bf Proof:}}{}
\newcommand{\ignore}[1]{}

\newcommand{\dd}{\mathrm{d}}

\newcommand{\EE}{\mathrm{E}}

\newcommand{\Real}{\mathbb{R}}

\newcommand{\fhat}{\hat{f}}

\newcommand{\fstar}{f^*}

\newcommand{\calB}{\mathcal{B}}

\newcommand{\calF}{\mathcal{F}}
\newcommand{\calG}{\mathcal{G}}
\newcommand{\calH}{\mathcal{H}}

\newcommand{\calN}{\mathcal{N}}

\newcommand{\calX}{\mathcal{X}}

\newcommand{\scrE}{\mathscr{E}}

\newcommand{\Eqref}[1]{Eq.~{\eqref{#1}}}

\newcommand{\kmin}{\kappa_M}

\newcommand{\hnorm}[1]{\|_{\calH_{#1}}}

\newcommand{\lambdaone}{{\lambda_1^{(n)}}}

\newcommand{\lambdatmp}{{\lambda}}

\newcommand{\Uns}{U_{n,s}}

\newcommand{\LPi}{L_2(\Pi)}

\newcommand{\calHlp}{\calH_{\ell_p}}
\newcommand{\calHl}[1]{\calH_{\ell_{#1}}}
\newcommand{\repH}{\widetilde{\calH}}
\newcommand{\repk}{\widetilde{k}}

\newcommand{\bmid}{~\Big |~}

\newcommand{\ellp}{$\ell_p$}
\newcommand{\kminrho}{\kmin}

\newcommand{\calHtot}{\calH^{\oplus M}}

\newtheorem{Theorem}{Theorem}

\newtheorem{Lemma}[Theorem]{Lemma}
\newtheorem{Proposition}[Theorem]{Proposition}
\newtheorem{Corollary}[Theorem]{Corollary}
\newtheorem{Assumption}{Assumption}

\newcounter{assump}
\renewcommand{\theassump}{\arabic{assump}} 
\newcommand{\Assump}[1][]{{\refstepcounter{assump}{#1} (A\theassump)}}

\usepackage[dvipdfm]{graphicx}

\allowdisplaybreaks

\title{Fast Learning Rate of $\ell_p$-MKL and its Minimax Optimality}
\author{Taiji Suzuki \\
Department of Mathematical Informatics, \\
The University of Tokyo,\\
7-3-1 Hongo, Bunkyo-ku, Tokyo 113-8656, JAPAN \\
\texttt{\small t-suzuki@mist.i.u-tokyo.ac.jp}
}
\date{January, 2010}

\begin{document}

\maketitle
\begin{abstract}
In this paper, we give a new sharp generalization bound of \ellp-MKL
which is a generalized framework of multiple kernel learning (MKL) and imposes $\ell_p$-mixed-norm regularization instead of $\ell_1$-mixed-norm regularization.
We utilize {\it localization techniques} to obtain the sharp learning rate. 
The bound is characterized by the decay rate of the eigenvalues of the associated kernels. 
A larger decay rate gives a faster convergence rate.
Furthermore, we give the minimax learning rate on the ball characterized by \ellp-mixed-norm in the product space. 
Then we show that our derived learning rate of \ellp-MKL achieves the minimax optimal rate on the \ellp-mixed-norm ball.
\end{abstract}

\section{Introduction}
Multiple Kernel Learning (MKL) proposed by \cite{JMLR:Lanckriet+etal:2004} is one of the most promising methods that adaptively select the kernel function
in supervised kernel learning.
Kernel method is widely used and several studies have supported its usefulness \citep{book:Schoelkopf+Smola:2002,Book:Taylor+Cristianini:2004}. 
However the performance of kernel methods critically relies on the choice of the kernel function.
Many methods have been proposed to deal with the issue of kernel selection. 
\cite{JMLR:Ong+etal:2005} studied hyperkrenels as a kernel of kernel functions. 
\cite{ICML:Argriou+etal:2006} considered DC programming approach to learn a mixture of kernels with continuous parameters (see also \cite{COLT:Andreas+etal:2005}).
Some studies tackled a problem to learn non-linear combination of kernels as in \cite{NIPS:Bach:2009,NIPS:Cortes+etal:nonlinear:2009,ICML:Varma+Babu:2009}. 
Among them, learning a linear combination of finite candidate kernels with non-negative coefficients is the most basic, fundamental and commonly used approach.
The seminal work of MKL by \citet{JMLR:Lanckriet+etal:2004} considered learning convex combination of candidate kernels. 
This work opened up the sequence of the MKL studies.
\cite{ICML:Bach+etal:2004} showed that MKL can be reformulated as a kernel version of the group lasso \citep{JRSS:YuanLin:2006}.
This formulation gives an insight that MKL can be described as a $\ell_1$ regularized learning method.
As a generalization of MKL, \ellp-MKL that imposes $\ell_p$-mixed-norm regularization 
($\sum_{m=1}^M \|f_m\hnorm{m}^p$ with $p\geq 1$) has been proposed \citep{JMLR:MicchelliPontil:2005,NIPS:Marius+etal:2009},
where $\{\calH_m\}_{m=1}^M$ are $M$ reproducing kernel Hilbert spaces (RKHSs) and $f_m \in \calH_m$.
\ellp-MKL includes the original MKL 
as a special case of $\ell_1$-MKL. 
One recent perception is that \ellp-MKL with $p>1$ shows better performances than $\ell_1$-MKL in several situations \citep{NIPS:Marius+etal:2009,UAI:Cortes+etal:2009}.
To justify the usefulness of $\ell_p$-MKL, a few papers have given theoretical analyses of $\ell_p$-MKL \citep{UAI:Cortes+etal:2009,ICML:Cortes+etal:gbound:2010,arXiv:Marius+etal:2010}. 
In this paper, we give a new faster learning rate of $\ell_p$-MKL utilizing the {\it localization techniques} \citep{Book:VanDeGeer:EmpiricalProcess,LocalRademacher,BartlettConvexity,Koltchinskii},
and show our learning rate is optimal in a sense of minimaxity. 
This is the first attempt to show the fast localized learning rate for \ellp-MKL.

In the pioneering paper of \citet{JMLR:Lanckriet+etal:2004}, 
a convergence rate of MKL is given as $\sqrt{\frac{M}{n}}$, where $M$ is the number of given kernels and $n$ is the number of samples. 
\cite{COLT:Srebro+BenDavid:2006} gave improved learning bound utilizing the pseudo-dimension of the given kernel class. 
\cite{COLT:Ying+Campbell:2009} gave a convergence bound utilizing Rademacher chaos
and gave some upper bounds of the Rademacher chaos utilizing the pseudo-dimension of the kernel class.
\cite{UAI:Cortes+etal:2009} presented a convergence bound for a learning method with $L_2$ regularization on the kernel weight.
\cite{ICML:Cortes+etal:gbound:2010} showed that the convergence rate of $\ell_1$-MKL is $\sqrt{\frac{\log(M)}{n}}$. 
They gave also the convergence rate of \ellp-MKL as $\frac{M^{1-\frac{1}{p}}}{\sqrt{n}}$ for $p>1$.
\cite{arXiv:Marius+etal:2010} gave a similar convergence bound with improve constants. 
\cite{ECML:Marius+etal:2010} generalized the bound to a variant of the elasticnet type regularization and 
widened the effective range of $p$ to all range of $p \geq 1$ while in the existing bounds $1\leq p \leq 2$ was imposed.
Our concern about the existing bounds is that all bounds introduced above are ``global'' bounds in a sense that  
the bounds are applicable to all candidates of estimators.
Consequently all convergence rate presented above are of order $1/\sqrt{n}$ with respect to the number $n$ of samples.
However, by utilizing the {\it localization} techniques including   
so-called local Rademacher complexity \citep{LocalRademacher,BartlettConvexity,Koltchinskii} and peeling device \citep{Book:VanDeGeer:EmpiricalProcess},
we can derive a faster learning rate.
Instead of uniformly bounding all candidates of estimators,  
the localized inequality focuses on a particular estimator such as empirical risk minimizer, 
thus can gives a sharp convergence rate.  

Localized bounds of MKL have been given mainly in sparse learning settings such as $\ell_1$-MKL or elasticnet type MKL \citep{NIPSWS:Taylor:2008,NIPSWS:ElastMKL:2009}.
The first localized bound of MKL is derived by \cite{COLT:Koltchinskii:2008} in the setting of $\ell_1$-MKL.
The second one was given by \cite{AS:Meier+Geer+Buhlmann:2009} who gave a near optimal convergence for elasticnet type regularization.
Recently \cite{AS:Koltchinskii+Yuan:2010} considered a variant of $\ell_1$-MKL and showed it achieves the minimax optimal convergence rate. 
All the localized convergence rates were considered in sparse learning settings. 
The localized fast learning rate of \ellp-MKL has not been addressed.

In this paper, we give a sharp convergence rate of \ellp-MKL utilizing the localization techniques.
Our bound also clarifies the relation between the convergence rate and the tuning parameter $p$.
The resultant convergence rate is $M^{1-\frac{2s}{p(1+s)}}n^{-\frac{1}{1+s}} R_p^{\frac{2s}{p(1+s)}}$ where $R_p= (\sum_{m=1}^M \|\fstar_m\hnorm{m}^p)^{\frac{1}{p}}$ determined 
by the true function $\fstar$ 
and $s$ ($0<s<1$) is a constant that represents the complexity of RKHSs and satisfies $0<s<1$.
The bound includes the bound of \cite{ICML:Cortes+etal:gbound:2010,arXiv:Marius+etal:2010} as a special case of $s\to 1$.
Finally, we show that the bound for \ellp-MKL achieves the minimax optimal rate in the ball with respect to \ellp-mixed-norm $\{f=\sum_{m=1}^M f_m \mid (\sum_{m=1}^M \|f_m\hnorm{m}^p)^{\frac{1}{p}} \leq R \}$. 
This indicates that \ellp-MKL is compatible with \ellp-mixed-norm. 

\section{Preliminary}
In this section we give the problem formulation, the notations and the assumptions for the convergence analysis of \ellp-MKL. 

\subsection{Problem Formulation}
Suppose that we are given $n$ i.i.d. samples $\{(x_i,y_i)\}_{i=1}^n$ distributed from a probability distribution $P$ on $\calX \times \Real$ that has the marginal distribution $\Pi$ on $\calX$. 
We are given $M$ reproducing kernel Hilbert spaces (RKHS) $\{\calH_m\}_{m=1}^M$ each of which is associated with a kernel $k_m$. 
\ellp-MKL ($p\geq 1$) fits a function $f = \sum_{m=1}^M f_m~(f_m\in \calH_m)$ to the data by solving the following optimization problem\footnote{
One might like to use $\sum_{m=1}^M \|f_m\hnorm{m}^p$ instead of $\left(  \sum_{m=1}^M \|f_m\hnorm{m}^p \right)^{\frac{2}{p}}$ as regularization.
However this difference does not matter because by adjusting the regularization parameter $\lambdaone$ there is a one-to-one correspondence between 
the solutions of both regularization types.}: 
\begin{align}
\fhat = \sum_{m=1}^M \fhat_m = &\mathop{\arg \min}_{f_m \in \calH_m~(m=1,\dots,M)}
\frac{1}{n}\sum_{i=1}^N \left(y_i- \sum_{m=1}^M f_m(x_i) \right)^2 +
\lambdaone \left(  \sum_{m=1}^M \|f_m\hnorm{m}^p \right)^{\frac{2}{p}}.
\label{eq:primalLp}
\end{align}
This is reduced to a finite dimensional optimization problem due to the representer theorem \citep{JMAA:KimeldorfWahba:1971}.
The problem is convex and thus there are efficient algorithms to solve that, e.g., \citet{NIPS:Marius+etal:2009,arXiv:Marius+etal:2010} and \citet{NIPS:Vishwanathan+etal:2010}.
In this paper, we focus on the regression problem (the squared loss).
However the discussion presented here can be generalized to Lipschitz continuous and strongly convex losses \citep{BartlettConvexity}. 

Sometimes the regularization of \ellp-MKL for $1\leq p \leq 2$ is imposed in terms of the kernel weight as 
\begin{align}
\mathop{\min}_{\theta\in \Real^M,~f \in \calH_{k_{\theta}}}~&
\frac{1}{n}\sum_{i=1}^N \left(y_i- f(x_i) \right)^2 +
\lambdaone  \|f\|_{\calH_{k_{\theta}}}^2 
~~~~\mathrm{s.t.}&  k_{\theta} = \sum_{m=1}^M \theta_m k_m,~\sum_{m=1}^M \theta_m^{\frac{p}{2-p}} = 1,~\theta_m \geq 0,
\label{eq:constraintLp}
\end{align}
where $\calH_{k_{\theta}}$ is the RKHS corresponding to the kernel $k_{\theta}$. 
However these two formulations are completely same, that is, we obtain the same resultant solution in both formulations 
(see Lemma 25 of \cite{JMLR:MicchelliPontil:2005} and \cite{NIPSWS:RegStrMKL:2010} for details).
Moreover our formulation \eqref{eq:primalLp} also covers the situation of $p>2$ while the kernel weight constraint formulation is restricted to $1\leq p \leq 2$.



\subsection{Notations and Assumptions}

Here, we prepare notations and conditions that are used in the analysis.  

Let $\calHtot = \calH_1 \oplus \dots \oplus \calH_M$. 
Throughout the paper, we assume the following technical conditions 
(see also \citep{JMLR:BachConsistency:2008}). 
\begin{Assumption}{\bf(Basic Assumptions)}\ 
\begin{enumerate}
\item[{\rm \Assump{\label{ass:truenoise}}}]
There exists $\fstar = (\fstar_1,\dots,\fstar_M) \in \calHtot$
such that $\EE[Y|X] = \fstar(X) = \sum_{m=1}^M \fstar_m(X)$,
and the noise $\epsilon := Y - \fstar(X)$ 
is bounded as $|\epsilon| \leq L$.
\item[{\rm \Assump{\label{ass:kernelbound}}}]
For each $m=1,\dots,M$, $\calH_m$ is separable (with respect to the RKHS norm) and $\sup_{X\in \calX} |k_m(X,X)| < 1$.
\end{enumerate}
\end{Assumption}
The first assumption in (A1) ensures the model $\calHtot$ is correctly specified, 
and the technical assumption $|\epsilon| \leq L$ allows $\epsilon f$ to be Lipschitz continuous with respect to $f$.
The noise boundedness can be relaxed to unbounded situation as in~\citep{arXiv:Raskutti+Martin:2010}, but
we don't pursue that direction for simplicity.
  

Due to Mercer's theorem, 
there are an orthonormal system $\{\phi_{k,m}\}_{k,m}$ in $L_2(\Pi)$
and the spectrum $\{\mu_{k,m}\}_{k,m}$
such that $k_m$ has the following spectral representation: 
\begin{equation}
k_m(x,x') = \sum_{k=1}^{\infty} \mu_{k,m} \phi_{k,m}(x) \phi_{k,m}(x'). 
\label{eq:spectralRepre}
\end{equation}
By this spectral representation, the inner-product of RKHS can be expressed as 
$
\langle f_m ,g_m \rangle_{\calH_m} = \sum_{k=1}^{\infty} \mu_{k,m}^{-1} \langle f_m, \phi_{k,m} \rangle_{\LPi} \langle \phi_{k,m}, g_m \rangle_{\LPi},
$
for $f_m, g_m \in \calH_m$.

\begin{table}[t]
\centering
\caption{Summary of the constants we use in this article.}
\label{tab:constants}
\begin{tabular}{|c|l|}
\hline
$n$ & The number of samples.  \\ \hline
$M$ & The number of candidate kernels.  \\ \hline
$s$ & The spectral decay coefficient; see (A3). \\ \hline 
$\kmin$ & The smallest eigenvalue of the design matrix (see \Eqref{eq:defkmin}). \\ \hline
$R_p$ & The $\ell_p$-mixed-norm of the truth: $(\sum_{m=1}^M \|\fstar_m \hnorm{m}^p)^{\frac{1}{p}}$	. \\ \hline
\end{tabular}
\end{table}


Constants we use later are summarized in Table~\ref{tab:constants}.
\begin{Assumption}{\bf (Spectral Assumption)}
\label{eq:specass}
There exist $0 < s < 1$ and $0 < c$ such that 
\begin{flalign*}
\text{\rm \Assump}&& 
\mu_{k,m} \leq c k^{-\frac{1}{s}},~~~(1\leq \forall k, 1\leq \forall m \leq M), && 
\end{flalign*}
where $\{\mu_{k,m}\}_{k}$ is the spectrum of the kernel $k_m$ (see Eq.\eqref{eq:spectralRepre}).
\end{Assumption}
It was shown that the spectral assumption (A3) is equivalent to 
the classical covering number assumption~\citep{COLT:Steinwart+etal:2009}.
Recall that 
the $\epsilon$-covering number $N(\epsilon,\mathcal{B}_{\calH_m},\LPi)$ with respect to $\LPi$
is the minimal number of balls with radius $\epsilon$ needed to cover the unit ball $\mathcal{B}_{\calH_m}$ in $\calH_m$ \citep{Book:VanDerVaart:WeakConvergence}.
If the spectral assumption (A3) holds, there exists a constant $C$ that
depends only on $s$ and $c$ such that 
\begin{align}
\label{eq:coveringcondition}
\log N(\varepsilon,\mathcal{B}_{\calH_m},\LPi) \leq C \varepsilon^{-2 s},
\end{align}
and the converse is also true (see \citet[Theorem 15]{COLT:Steinwart+etal:2009} and \citet{Book:Steinwart:2008} for details).
Therefore, 
if $s$ is large, 
the RKHSs are regarded as ``complex'',
and if $s$ is small, the RKHSs are ``simple''.

Associated with the $\epsilon$-covering number, 
the {\it $i$-th entropy number} $e_i(\calH_m \to \LPi) $ is defined as the infimum over all $\varepsilon>0$ for which $N(\varepsilon,\mathcal{B}_{\calH_m},\LPi) \leq 2^{i-1}$.
If the spectral assumption (A3) holds, 
the relation \eqref{eq:coveringcondition} implies that the $i$-th entropy number is bounded as 
\begin{align}
\label{eq:entropycondition}
e_i(\calH_m \to \LPi) \leq C i^{- \frac{1}{2s}},
\end{align}
where $C$ is a constant.
To bound empirical process a bound of the entropy number with respect to the empirical distribution is needed.
The following proposition gives an upper bound of that (see Corollary 7.31 of \cite{Book:Steinwart:2008}, for example).
\begin{Proposition}
\label{prop:upperboundofe}
If there exists constants $0<s<1$ and $C \geq 1$ such that 
$e_i(\calH_m \to \LPi) \leq C i^{- \frac{1}{2s}}$, then  there exists a constant $c_s > 0$ only depending on $s$ such that 
\begin{align*}
\EE_{D_n \sim \Pi^n}[e_i(\calH_m \to L_2(D_n))] \leq c_s C (\min(i,n))^{\frac{1}{2s}} i^{-\frac{1}{s}},
\end{align*}
in particular 
$
\EE_{D_n \sim \Pi^n}[e_i(\calH_m \to L_2(D_n))] \leq c_s C  i^{-\frac{1}{2s}}.
$
\end{Proposition}
An important class of RKHSs where $s$ is known is Sobolev space.
(A3) holds with $s=\frac{d}{2m}$ for Sobolev space of $m$-times continuously differentiability on the Euclidean ball of $\Real^d$ \citep[Theorem 2.7.1]{Book:VanDerVaart:WeakConvergence}.  
Moreover, for $m$-times differentiable kernels on a closed Euclidean ball in $\Real^d$, 
that holds for $s=\frac{d}{2m}$ \citep[Theorem 6.26]{Book:Steinwart:2008}. 
According to \citet{JC:Zhou:2002}, for Gaussian kernels with compact support, that holds for arbitrary small $0<s$.
The entropy number of Gaussian kernels with {\it unbounded} support is described in Theorem 7.34 of \citet{Book:Steinwart:2008}.

Let $\kmin$ be defined as follows:
\begin{align}
\label{eq:defkmin}
\kmin &:= \sup\left\{\kappa \geq 0 ~\Big |~ 
\kappa \leq  \frac{\|\sum_{m=1}^M f_m\|_{\LPi}^2}{\sum_{m=1}^M \|f_m\|_{\LPi}^2} ,~\forall f_m \in \calH_m~(m=1,\dots,M)\right\}. 
\end{align}
$\kmin$ represents the correlation of RKHSs. 
We assume all RKHSs are not completely correlated to each other.
\begin{Assumption}{\bf (Incoherence Assumption)}
$\kappa_M$ is strictly bounded from below; there exists a constant $C_0>0$ such that 
\begin{flalign*}
\text{\rm \Assump} 
 && 0 <C_0^{-1} < \kmin. && 
\end{flalign*}
\end{Assumption}
This condition is motivated by the {\it incoherence condition} \citep{COLT:Koltchinskii:2008,AS:Meier+Geer+Buhlmann:2009}
considered in sparse MKL settings. This ensures the uniqueness of the decomposition $\fstar = \sum_{m=1}^M \fstar_m$ of the ground truth.
\cite{JMLR:BachConsistency:2008} also assumed this condition to show the consistency of $\ell_1$-MKL.

Finally we give a technical assumption with respect to $\infty$-norm.
\begin{Assumption}{\bf (Embedded Assumption)}
\label{ass:linfbound}
Under the Spectral Assumption ($s$), there exists a constant $C_1>0$ such that 
\begin{flalign*}
\text{\rm \Assump }
&& \|f_m\|_{\infty} \leq C_1 \|f_m\hnorm{m}^{1-s} \|f_m\|_{\LPi}^s. &&
\end{flalign*}
\end{Assumption}
This condition is met when the RKHSs are continuously embedded in a Besov space $B_{2,1}^{sm}(X)$ where $s=\frac{d}{2m}$, $d$ is the dimension of the input space $\calX$ and 
$m$ is the smoothness of the Besov space.
For example, the RKHSs of Gaussian kernels can be embedded in all Sobolev spaces,
and therefore the condition (A5) seems rather common and practical.
More generally, there is a clear characterization of the condition (A5) in terms of {\it real interpolation of spaces}. 
One can find 
detailed and formal discussions of interpolations in \citet{COLT:Steinwart+etal:2009}, and 
Proposition 2.10 of \citet{Book:Bennett+Sharpley:88} gives the necessary and sufficient condition for the condition (A5).

\section{Convergence Rate of $\ell_p$-MKL}
Here we derive the convergence rate of the estimator $\fhat$. 
We suppose that the number of kernels $M$ can increase along with the number of samples $n$.
The motivation of our analysis is summarized as follows:
\begin{itemize}
\item Deriving a sharp convergence rate utilizing localization techniques.
\item Clarifying the relation between the norm $(\sum_{m=1}^M \|\fstar_m\hnorm{m}^p)^{\frac{1}{p}} $ of the truth and the generalization bound.
\end{itemize}

Now we define 
\[
\eta(t) := \eta_{n}(t) = \max(1,\sqrt{t},t/\sqrt{n}),
\]
and for a given positive real $\lambda$ we define 
\begin{align}
\zeta_n := 2 \left(\sqrt{\frac{M\log(M)}{n}}\vee \frac{\lambdatmp^{-\frac{s}{2}} M^{\frac{1+s}{2} -\frac{s}{p}}}{\sqrt{n}} \vee  
\frac{M^{\frac{1+4s-s^2}{2(1+s)} - \frac{s(3-s)}{p(1+s)}}\lambdatmp^{-\frac{s(3-s)}{2(1+s)}}}{n^{\frac{1}{1+s}}} \right).
\label{eq:defzetan}
\end{align}
Then we obtain the following convergence rate.
\begin{Theorem}
\label{th:convergencerateofLpMKL}
Suppose $\lambdaone > 0$ and let $\lambdatmp$ in the definition \eqref{eq:defzetan} of $\zeta_n$ 
be $\lambdatmp = \lambdaone$.
Then there exists a constant $\psi_s$ depending $L,s,c,C_1$ such that 
for all $n$ and $t'(>0)$ that satisfy  $\frac{\log(M)}{\sqrt{n}}\leq 1$ and 
\begin{align*}
\frac{\psi_s \sqrt{n} \zeta_n^2}{\kmin} \eta(t') \leq 1,
\end{align*}
the solution of \ellp-MKL given in \Eqref{eq:primalLp} for arbitrary real $p\geq 1$ satisfies  
\begin{align}
& \|\fhat - \fstar\|_{\LPi}^2 
\leq 
 \frac{\psi_s^2}{ \kminrho}\zeta_n^2 \eta(t)^2 
+ 
\frac{8}{3} \lambdaone \left( \sum_{m=1}^M \|\fstar_m \hnorm{m}^p \right)^{\frac{2}{p}},
\label{eq:theMainBound}
\end{align}
with probability $1- \exp(- t) - \exp(-t')$ for all $t \geq 1$.
\end{Theorem}
The proof will be given in Appendix \ref{sec:ProofMainTh}.

Let $R_p := \left( \sum_{m=1}^M \|\fstar_m \hnorm{m}^p \right)^{\frac{1}{p}}$.
Suppose that $n$ is sufficiently large compared with $M$ and $R_p$ 
($n\geq M^{\frac{2}{p}} R_p^{-2} (\log M)^{\frac{1+s}{s}}$ and $n \geq (R_p/M^{\frac{1}{p}})^{\frac{4s}{1-s}}$). 
Then the regularization parameter $\lambdaone$ that achieves the minimum of the RHS of the bound \eqref{eq:theMainBound} is given by 
$$
\lambdaone = n^{-\frac{1}{1+s}} M^{1-\frac{2s}{p(1+s)}} R_p^{-\frac{2}{1+s}},
$$
up to constant.
Then the convergence rate of $\|\fhat - \fstar\|_{\LPi}^2 $ becomes 
\begin{align}
\|\fhat - \fstar\|_{\LPi}^2 = &\mathcal{O}_p\Big( n^{-\frac{1}{1+s}} M^{1-\frac{2s}{p(1+s)}} R_p^{\frac{2s}{1+s}} + \frac{M\log(M)}{n} 
+ n^{-\frac{1}{1+s} - \frac{(s-1)^2}{(1+s)^2}} M^{1- \frac{2s(3-s)}{p(1+s)^2}} 
R_p^{\frac{2s(3-s)}{(1+s)^2}} \Big). 
\label{eq:convergenceRateComplex}
\end{align}
Under the condition $n\geq M^{\frac{2}{p}} R_p^{-2} (\log M)^{\frac{1+s}{s}}\vee (R_p/M^{\frac{1}{p}})^{\frac{4s}{1-s}}$, 
the leading term is the first term,
and thus we have 
\begin{equation}
\label{eq:convergenceRateSimple}
\|\fhat - \fstar\|_{\LPi}^2 = \mathcal{O}_p \left(n^{-\frac{1}{1+s}} M^{1-\frac{2s}{p(1+s)}} R_p^{\frac{2s}{1+s}}\right).
\end{equation}
Note that as the complexity $s$ of RKHSs becomes small the convergence rate becomes fast.  
It is known that $n^{-\frac{1}{1+s}}$ is the minimax optimal learning rate for single kernel learning.
The derived rate of \ellp-MKL is obtained by multiplying a coefficient depending on $M$ and $R_p$ to the optimal rate of single kernel learning.
To investigate the dependency of $R_p$ to the learning rate, let us consider two extreme settings, i.e., 
sparse setting $(\|\fstar_m\hnorm{m})_{m=1}^M = (1,0,\dots,0)$ and dense setting $(\|\fstar_m\hnorm{m})_{m=1}^M = (1,\dots,1)$
as in \cite{arXiv:Marius+etal:2010}.
\begin{itemize}
\item $(\|\fstar_m\hnorm{m})_{m=1}^M = (1,0,\dots,0)$: $R_p=1$ for all $p$. 
Therefore the convergence rate $n^{-\frac{1}{1+s}} M^{1-\frac{2s}{p(1+s)}}$ is fast for small $p$ and the minimum is achieved at $p=1$.
This means that $\ell_1$ regularization is preferred for sparse truth.  
\item $(\|\fstar_m\hnorm{m})_{m=1}^M = (1,\dots,1)$: $R_p = M^{\frac{1}{p}}$, thus the convergence rate is $M n^{-\frac{1}{1+s}}$ for all $p$. 
Interestingly for dense ground truth, there is no dependency of the convergence rate on the parameter $p$. 
That is, the convergence rate is $M$ times the optimal learning rate of single kernels learning ($n^{-\frac{1}{1+s}}$) for all $p$.
This means that for the dense settings, the complexity of solving MKL problem is equivalent to that of solving $M$ single kernel learning problems.
\end{itemize}


\subsection{Comparison with existing bounds}
Here we compare the bound we derived with the existing bounds. 
Let $\calHlp(R)$ be the $\ell_p$-mixed norm ball with radius $R$ defined as follows:
$$
\calHlp(R)  := \left\{f = \sum_{m=1}^M f_m \bmid \left( \sum_{m=1}^M \|f_m\hnorm{m}^p \right)^{\frac{1}{p}} \leq R \right\}.
$$
The bounds by \cite{ICML:Cortes+etal:gbound:2010} and \cite{ECML:Marius+etal:2010,arXiv:Marius+etal:2010} are most relevant to our results.
Roughly speaking, their bounds are given as
\begin{align}
R(f) \leq \widehat{R}(f) + C \frac{M^{1-\frac{1}{p}}\vee \sqrt{\log(M)}}{\sqrt{n}}R~~~\text{for all $f \in \calH_{\ell_p}(R)$},  
\label{eq:CortesBound}
\end{align}
where $R(f)$ and $\widehat{R}(f)$ is the population risk and the empirical risk.
First observation is that the bounds by \cite{ICML:Cortes+etal:gbound:2010} and \cite{arXiv:Marius+etal:2010} are restricted to the situation $1\leq p \leq 2$ 
because their analysis is based on the kernel weight constraint formulation \eqref{eq:constraintLp}.
On the other hand, our analysis ans that of \cite{ECML:Marius+etal:2010} covers all $p \geq 1$.
Second, since our bound is specialized to the regularized risk minimizer $\hat{f}$ defined at \Eqref{eq:primalLp}
while the existing bound \eqref{eq:CortesBound} is applicable to all $f \in \calH_{\ell_p}(R)$,
our bound is sharper than theirs. 
To see this, suppose that  $1\leq p \leq 2$ and $n^{-\frac{1}{2}} M^{1-\frac{1}{p}} \leq 1$ (which means the bound \eqref{eq:CortesBound} makes sense), 
then we have $n^{-\frac{1}{1+s}} M^{1-\frac{2s}{p(1+s)}} \leq n^{-\frac{1}{2}} M^{1-\frac{1}{p}}$.  
For the situation of $p \geq 2$, we have also $n^{-\frac{1}{1+s}} M^{1-\frac{2s}{p(1+s)}} \leq n^{-\frac{1}{2}} M^{1-\frac{1}{p}}$ for $n\geq M^{\frac{2}{p}}$. 
Moreover we should note that $s$ can be large as long as Spectral Assumption (A3) is satisfied. 
Thus the bound \eqref{eq:CortesBound} is recovered by our analysis by approaching $s$ to 1.

The results by \cite{COLT:Koltchinskii:2008,AS:Meier+Geer+Buhlmann:2009,AS:Koltchinskii+Yuan:2010} are also related to ours 
in terms of the proof techniques.
Their analyses and ours utilize the localization techniques to obtain {\it fast localized learning rate}, 
in contrast to the global bound of \cite{ICML:Cortes+etal:gbound:2010,ECML:Marius+etal:2010,arXiv:Marius+etal:2010}. 
However all those localized bounds are considered on a sparse learning settings such as $\ell_1$ and elasticnet regularizations.
Hence their frameworks are rather different from ours.


\section{Lower bound of learning rate}
In this section, we show that the derived learning rate achieves the minimax-learning rate on $\calHlp(R)$.  

We derive the minimax learning rate in a simpler situation.
First we assume that each RKHS is same as others.
That is, the input vector is decomposed into $M$ components like $x=(x^{(1)},\dots,x^{(M)})$ where $\{x^{(m)}\}_{m=1}^M$ are $M$ i.i.d. copies of a random variable $\tilde{X}$,
and $\calH_m = \{f_m \mid f_m(x) = f_m(x^{(1)},\dots,x^{(M)}) = \tilde{f}_m(x^{(m)}),~\tilde{f}_m \in \repH \}$ where $\repH$ is an RKHS shared by all $\calH_m$.
Thus $f\in \calHtot$ is decomposed as  
$f(x) = f(x^{(1)},\dots,x^{(M)}) = \sum_{m=1}^M \tilde{f}_m(x^{(m)})$ where each $\tilde{f}_m$ is a member of the common RKHS $\repH$. 
We denote by $\repk$ the kernel associated with the RKHS $\repH$.

In addition to the condition about the upper bound of spectrum (Spectral Assumption (A3)), 
we assume that the spectrum of all the RKHSs $\calH_m$ have the same lower bound of polynomial rate. 
\begin{Assumption}{\bf (Strong Spectral Assumption)}
\label{eq:specass2}
There exist $0 < s < 1$ and $0<c,c'$ such that 
\begin{flalign*}
\text{\rm(A6)} &&
c' k^{-\frac{1}{s}} \leq \tilde{\mu}_{k} \leq c k^{-\frac{1}{s}},~~~(1\leq \forall k),&&
\end{flalign*}
where $\{\tilde{\mu}_{k}\}_{k}$ is the spectrum of the kernel $\tilde{k}$. 
In particular, the spectrum of kernels $k_m$ also satisfy $\mu_{k,m} \sim k^{-\frac{1}{s}}~(\forall k,m)$.
\end{Assumption}
As discussed just after Assumption \ref{eq:specass2}, this means that the covering number of $\repH$ satisfies
\begin{align*}
\calN(\varepsilon,\mathcal{B}_{\repH},\LPi) \sim  \varepsilon^{-2 s},
\end{align*}
where $\mathcal{B}_{\repH}$ is the unit ball of $\repH$ (see \citet[Theorem 15]{COLT:Steinwart+etal:2009} and \citet{Book:Steinwart:2008} for details).
Without loss of generality, we may assume that 
\begin{align*}
\EE[f(\tilde{X})] = 0~~~(\forall f \in \repH).
\end{align*}
Since each $f_m$ receives i.i.d. copy of $\tilde{X}$, $\calH_m$s are orthogonal to each other:
\begin{align*}
\EE[f_m(X) f_{m'}(X)] = \EE[\tilde{f}_m(X^{(m)}) \tilde{f}_{m'}(X^{(m')})] =0~~~(\forall f_m \in \calH_m,~\forall f_{m'} \in \calH_{m'},~1\leq \forall m \neq m' \leq M).
\end{align*}
We also assume that the noise $\{\epsilon_i\}_{i=1}^n$ is an i.i.d. normal sequence with standard deviation $\sigma>0$.  

Under the assumptions described above, then we have the following minimax $\LPi$-error.
\begin{Theorem}
\label{th:LowerBounds}
For a given $0<R$, the minimax-learning rate on $\calHlp(R)$ is lower bounded as 
\begin{align*}
\min_{\hat{f}} \max_{f^* \in \calHlp(R)} \EE\left[ \|\hat{f} - f^* \|_{\LPi}^2\right] \geq C n^{-\frac{1}{1+s}} M^{1-\frac{2s}{p(1+s)}} R^{\frac{2s}{1+s}},
\end{align*}
where $\inf$ is taken over all measurable functions of $n$ samples $\{(x_i,y_i)\}_{i=1}^n$.
\end{Theorem}
The proof will be given in Appendix \ref{sec:proofOfMinimax}.
One can see that the convergence rate derived in Theorem \ref{th:convergencerateofLpMKL} and \Eqref{eq:convergenceRateSimple} 
achieves the lower bound of Theorem \ref{th:LowerBounds}.  
Thus our bound is tight.
Interestingly, the learning rate \eqref{eq:convergenceRateSimple} of \ellp-MKL and the minimax learning rate 
on $\ell_p$-mixed-norm ball coincide at the common $p$. 
This means that the \ellp-mixed-norm regularization is well suited to make the estimator included in the \ellp-mixed-norm ball.

\section{Conclusion and Discussion}
We have shown a sharp optimal learning rate of \ellp-MKL by utilizing the localization techniques.
Our bound is sharper than existing bounds and achieves the minimax learning rate under the Spectral Assumption (A3).

There still remain important future works. 
The bound given in \Eqref{eq:convergenceRateSimple} becomes smaller 
as $p$ becomes smaller since $R_p/M^{\frac{1}{p}}$ decreases as $p \searrow 1$.
That is, according to the theoretical result, \ellp-MKL shows the best performance at $p=1$ 
despite the disappointing results of $p=1$ reported by some numerical experiments.  
This concern was also pointed out by \cite{ICML:Cortes+etal:gbound:2010}. 
It is an important future work to theoretically clarify why \ellp-MKL with $p>1$ works well in some real situations.
The second interesting future work is about the $\frac{M \log(M)}{n}$ term appeared in the bound \Eqref{eq:convergenceRateComplex}. 
Because of this term, our bound is $O(M\log(M))$ with respect to $M$ while in the existing work that is $O(M^{1-\frac{1}{p}})$. 
It is an interesting issue to clarify whether the term $\frac{M \log(M)}{n}$ can be replaced by other tighter bounds or not.
To do so, it might be useful to precisely estimate the covering number of $\calH_{\ell_p}(R)$.

\section*{Acknowledgement}
We would like to thank Ryota Tomioka and Masashi Sugiyama for suggestive discussions. 
This work was partially supported by MEXT Kakenhi 22700289.

\appendix 
\section{Proof of Theorem \ref{th:convergencerateofLpMKL}}
\label{sec:ProofMainTh}
Before we show Theorem \ref{th:convergencerateofLpMKL}, we prepare several lemmas. 
The following two propositions are key for localization. 
Let $\{\sigma_i\}_{i=1}^n$ be i.i.d. Rademacher random variables, i.e., $\sigma_i \in \{\pm 1\}$ and $P(\sigma_i=1) = P(\sigma_i=-1)=\frac{1}{2}$.  
\begin{Proposition}{\rm \bf \cite[Theorem 7.16]{Book:Steinwart:2008}}
\label{prop:localFmsBound}
Let $\calB_{\sigma,a,b} \subset \calH_m$ be a set such that $\calB_{\sigma,a,b} = \{ f_m \in \calH_m \mid \|f_m\|_{\LPi}\leq \sigma, \|f_m\hnorm{m} \leq a, \|f_m\|_{\infty} \leq b\} $.
Assume that there exist constants $0<s<1$ and $0 < \tilde{c}_s$ such that 
\begin{align*}
\EE_{D_n}[ e_i(\calH_m \to L_2(D_n))] \leq \tilde{c}_s i^{-\frac{1}{2s}}.
\end{align*}
Then there exists a constant $C_s'$ depending only $s$ such that 
\begin{align}
\EE\left[\sup_{f_m \in \calB_{\sigma,a,b}} \left| \frac{1}{n}\sum_{i=1}^n \sigma_i f_m(x_i) \right|\right] \leq 
C_s' \left( \frac{ \sigma^{1-s} (\tilde{c}_s a)^s}{\sqrt{n}} \vee (\tilde{c}_s a)^{\frac{2s}{1+s}} b^{\frac{1-s}{1+s}} n^{-\frac{1}{1+s}} \right). 
\label{eq:localFmsBound}
\end{align}
\end{Proposition}

\begin{Proposition}{\rm \bf (Talagrand's Concentration Inequality \citep{Talagrand2,BousquetBenett})}
\label{prop:TalagrandConcent}
Let $\calG$ be a function class on $\calX$ that is separable with respect to $\infty$-norm, and 
$\{x_i\}_{i=1}^n$ be i.i.d. random variables with values in $\calX$.
Furthermore, let $B\geq 0$ and $U\geq 0$ be 
$B := \sup_{g \in \calG} \EE[(g-\EE[g])^2]$ and $U := \sup_{g \in \calG} \|g\|_{\infty}$,
then there exists a universal constant $K$ such that, for $Z := \sup_{g\in \calG}\left|\frac{1}{n} \sum_{i=1}^n g(x_i) - \EE[g] \right|$, we have
\begin{align*}
P\left( Z \geq K\left[\EE[Z] + \sqrt{\frac{B t}{n}} + \frac{U t}{n} \right] \right) \leq e^{-t}.
\end{align*}
\end{Proposition}

Let $\lambdatmp >0$ be an arbitrary positive real. 
We determine $\Uns(f_m)$ as follows:
\begin{align*}
\Uns(f_m) := &\left(\sqrt{\frac{M\log(M)}{n}}\vee 
\frac{\lambdatmp^{-\frac{s}{2}} M^{\frac{1-s}{2} + s(1-\frac{1}{p})}}{\sqrt{n}} \vee 
\frac{\lambdatmp^{-\frac{(3-s)}{2(1+s)}} M^{\frac{1+4s -s^2}{2(1+s)} - \frac{s(3-s)}{p(1+s)}}}{n^{\frac{1}{1+s}}}  \right) \times \\
&
\left(\frac{\|f_m\|_{\LPi}}{\sqrt{M}} + \frac{\lambdatmp^{\frac{1}{2}} \|f_m\hnorm{m}}{M^{1-\frac{1}{p}}}\right).
\end{align*} 
It is easy to see $\Uns(f_m)$ is an upper bound of the quantity  
$\frac{\|f_m\|_{\LPi}^{1-s}\|f_m \hnorm{m}^s}{\sqrt{n}} \vee 
\frac{\|f_m \|_{\LPi}^{\frac{(1-s)^2}{1+s}} \|f_m \hnorm{m}^{\frac{s(3-s)}{1+s}}}{n^{\frac{1}{1+s}}}$
(this corresponds to the RHS of \Eqref{eq:localFmsBound})
because 
\begin{align}
\frac{\|f_m\|_{\LPi}^{1-s}\|f_m \hnorm{m}^s}{\sqrt{n}} 
&
=\frac{\lambdatmp^{-\frac{s}{2}} M^{\frac{1-s}{2}+s(1-\frac{1}{p})}}{\sqrt{n}} 
\left(\frac{\|f_m\|_{\LPi}}{\sqrt{M}}\right)^{1-s} \left(\frac{\lambdatmp^{\frac{1}{2}}\|f_m \hnorm{m}}{M^{1-\frac{1}{p}}}\right)^s \notag \\
&\leq 
\frac{\lambdatmp^{-\frac{s}{2}} M^{\frac{1-s}{2}+s(1-\frac{1}{p})}}{\sqrt{n}} 
\left(\frac{\|f_m\|_{\LPi}}{\sqrt{M}} + \frac{\lambdatmp^{\frac{1}{2}}\|f_m \hnorm{m}}{M^{1-\frac{1}{p}}}\right), 
\label{eq:sdecomptwo}
\end{align}
where we used Young's inequality in the last line, and similarly we obtain
\begin{align*}
\frac{\|f_m \|_{\LPi}^{\frac{(1-s)^2}{1+s}} \|f_m \hnorm{m}^{\frac{s(3-s)}{1+s}}}{n^{\frac{1}{1+s}}}
&\leq 
\frac{\lambdatmp^{-\frac{(3-s)}{2(1+s)}} M^{\frac{(1-s)^2}{2(1+s)} + (1-\frac{1}{p})\frac{s(3-s)}{1+s}}}{n^{\frac{1}{1+s}}} 
\left(\frac{\|f_m\|_{\LPi}}{\sqrt{M}} + \frac{\lambdatmp^{\frac{1}{2}}\|f_m \hnorm{m}}{M^{1-\frac{1}{p}}}\right).
\end{align*}
Using Propositions \ref{prop:TalagrandConcent} and \ref{prop:localFmsBound}, 
we obtain the following ratio type uniform bound.
\begin{Lemma}
\label{lemm:uniformratiobound}
Under the Spectral Assumption (Assumption \ref{eq:specass}) and the Embedded Assumption (Assumption \ref{ass:linfbound}), 
there exists a constant $C_s$ depending only on $s$, $c$ and $C_1$ such that 
\begin{align}
\EE\left[\sup_{f_m \in \calH_m: \|f_m \hnorm{m}=1} 
\frac{|\frac{1}{n}\sum_{i=1}^n \sigma_i f_m(x_i)|}{\Uns(f_m)} \right] 
\leq C_s.
\end{align}
\end{Lemma}

\begin{proof} [Proof of Lemma \ref{lemm:uniformratiobound}]
Let $\calH_m(\sigma) := \{ f_m \in \calH_m \mid \|f_m\hnorm{m} = 1, \|f_m\|_{\LPi} \leq \sigma \}$ and $z = 2^{1/s} > 1$.
Define $\tau := \lambdatmp^{\frac{s}{2}}M^{\frac{1}{2}-\frac{1}{p}}$. Then 
by combining Propositions \ref{prop:upperboundofe} and \ref{prop:localFmsBound} with Assumption \ref{ass:linfbound}, we have 
\begin{align*}
&
\EE\left[\sup_{f_m \in \calH_m: \|f_m \hnorm{m}=1} 
\frac{|\frac{1}{n}\sum_{i=1}^n \sigma_i f_m(x_i)|}{\Uns(f_m)} \right]  \\
\leq 
&
\EE\left[\sup_{f_m \in \calH_m(\tau) } 
\frac{|\frac{1}{n}\sum_{i=1}^n \sigma_i f_m(x_i)|}{\Uns(f_m)} \right] 
+
\sum_{k= 1}^{\infty}
\EE\left[\sup_{f_m \in \calH_m(\tau z^k) \backslash \calH_m(\tau z^{k-1})} 
\frac{|\frac{1}{n}\sum_{i=1}^n \sigma_i f_m(x_i)|}{\Uns(f_m)} \right] \\
\leq 
&
C_s'  \frac{\frac{\tau^{1-s}\tilde{c}_s^s}{\sqrt{n}}}
{ \frac{\lambdatmp^{-\frac{s}{2}}M^{\frac{1+s}{2}-\frac{s}{p}}}{\sqrt{n}} \frac{\lambdatmp^{\frac{1}{2}}}{M^{1-\frac{1}{p}}}} 
\vee 
\frac{\frac{C_1^{\frac{1-s}{1+s}} \tau^{\frac{(1-s)^2}{1+s}}\tilde{c}_s^{\frac{2s}{1+s}}}{n^{\frac{1}{1+s}}}}
{ \frac{\lambdatmp^{-\frac{s(3-s)}{1+s}}M^{\frac{(1-s)^2}{2(1+s)}+(1-\frac{1}{p})\frac{s(3-s)}{1+s}}}{n^{\frac{1}{1+s}}} \frac{\lambdatmp^{\frac{1}{2}}}{M^{1-\frac{1}{p}}} } \\
&+
\sum_{k=1}^\infty
C_s' 
\frac{\frac{z^k(1-s) \tau^{1-s}\tilde{c}_s^s}{\sqrt{n}}}
{ \frac{\lambdatmp^{-\frac{s}{2}}M^{\frac{1+s}{2}-\frac{s}{p}}}{\sqrt{n}} \frac{\tau z^{k-1}}{\sqrt{M}}} 
\vee 
\frac{\frac{C_1^{\frac{1-s}{1+s}} z^{k\frac{(1-s)^2}{1+s}} \tau^{\frac{(1-s)^2}{1+s}}\tilde{c}_s^{\frac{2s}{1+s}}}{n^{\frac{1}{1+s}}}}
{ \frac{\lambdatmp^{-\frac{s(3-s)}{1+s}}M^{\frac{(1-s)^2}{2(1+s)}+(1-\frac{1}{p})\frac{s(3-s)}{1+s}}}{n^{\frac{1}{1+s}}} \frac{\tau z^{k-1}}{\sqrt{M}}} \\
& 
\leq  
C_s' \left(\tilde{c}_s^{s} \vee C_1^{\frac{1-s}{1+s}} \tilde{c}_s^{\frac{2s}{1+s}}\right)\left( 1 +  \sum_{k=1}^\infty z^{1 - ks} \vee z^{1 - k\frac{s(3-s)}{1+s}} \right) \\
&
= 
C_s'  \left(\tilde{c}_s^{s} \vee C_1^{\frac{1-s}{1+s}} \tilde{c}_s^{\frac{2s}{1+s}}\right)\left(1 +  \frac{z^{1-s}}{1-z^{-s}}\vee \frac{z^{1-\frac{s(3-s)}{1+s}}}{1-z^{-\frac{s(3-s)}{1+s}}} \right) .
\end{align*}
Thus by setting, $C_s = C_s'  \left(\tilde{c}_s^{s} \vee \tilde{c}_s^{\frac{2s}{1+s}}\right)\left(1 +  \frac{z^{1-s}}{1-z^{-s}}\vee \frac{z^{1-\frac{s(3-s)}{1+s}}}{1-z^{-\frac{s(3-s)}{1+s}}} \right)$,
we obtain the assertion.
\end{proof}

This lemma immediately gives the following corollary.
\begin{Corollary}
\label{cor:basicuniformcor}
Under the Spectral Assumption (Assumption \ref{eq:specass}) and the Embedded Assumption (Assumption \ref{ass:linfbound}),  
there exists a constant $C_s$ depending only on $s$ and $C$ such that 
\begin{align*}
\EE\left[\sup_{f_m \in \calH_m} 
\frac{|\frac{1}{n}\sum_{i=1}^n \sigma_i f_m(x_i)|}
{\Uns(f_m)}\right] 
\leq C_s. 
\end{align*}
\end{Corollary}

\begin{Lemma}
\label{lemm:basicuniformlemm}
If $\frac{\log(M)}{\sqrt{n}}\leq 1$, then under the Spectral Assumption (Assumption \ref{eq:specass}) and the Embedded Assumption (Assumption \ref{ass:linfbound})  
there exists a constant $\tilde{C}_s$ depending only on $s$, $c$, $C_1$ such that 
\begin{align*}
\EE\left[\max_m \sup_{f_m \in \calH_m} 
\frac{|\frac{1}{n}\sum_{i=1}^n \sigma_i f_m(x_i)|}{\Uns(f_m)} \right] 
\leq \tilde{C}_s. 
\end{align*}
\end{Lemma}

\begin{proof} [Proof of Lemma \ref{lemm:basicuniformlemm}]
First notice that the $\LPi$-norm and the $\infty$-norm of $\frac{\sigma_i f_m(x_i)}{\Uns(f_m)}$ can be evaluated by  
\begin{align}
&\left\| \frac{\sigma_i f_m(x_i)}{\Uns(f_m)} \right\|_{\LPi} = \frac{\left\|f_m\right\|_{\LPi}}{\Uns(f_m)} 
\leq
\left(\sqrt{\frac{\log(M)}{n}}\vee \frac{\lambdatmp^{-\frac{s}{2}} M^{\frac{-s}{2} + s(1-\frac{1}{p})}}{\sqrt{n}}\right)^{-1}
\leq \sqrt{\frac{n}{\log(M)}},
\label{eq:ratioLtwonormBound}
\\
&
\left\| \frac{\sigma_i f_m(x_i)}{\Uns(f_m)} \right\|_{\infty} =  \frac{\| f_m \|_{\infty}}{\Uns(f_m)} 
\leq  \frac{C_1 \|f_m \|_{\LPi}^{1-s} \|f_m \|_{\infty}^s }{\Uns(f_m)} \leq C_1 \sqrt{n}.
\label{eq:ratioinfnormBound}
\end{align}
The last inequality of \Eqref{eq:ratioinfnormBound} can be shown by using the relation \eqref{eq:sdecomptwo}.
Thus Talagrand's inequality implies 
\begin{align*}
&P\left(\max_m \sup_{f_m \in \calH_m} 
\frac{|\frac{1}{n}\sum_{i=1}^n \sigma_i f_m(x_i)|}{\Uns(f_m)}
\geq
K\left[C_s + \sqrt{\frac{t}{\log(M)}} + \frac{C_1 t}{\sqrt{n}}\right]
\right) \notag \\
\leq 
&
\sum_{m=1}^M 
P\left(\sup_{f_m \in \calH_m} 
\frac{|\frac{1}{n}\sum_{i=1}^n \sigma_i f_m(x_i)|}{\Uns(f_m)}
\geq
K\left[C_s + \sqrt{\frac{t}{\log(M)}} + \frac{C_1 t}{\sqrt{n}}\right]
\right) \notag 
\\ 
\leq
&
M e^{-t}.
\end{align*}
By setting $t \leftarrow t + \log(M)$, we obtain 
\begin{align*}
&P\left(\max_m \sup_{f_m \in \calH_m} 
\frac{|\frac{1}{n}\sum_{i=1}^n \sigma_i f_m(x_i)|}{\Uns(f_m)}
\geq
K\left[C_s + \sqrt{\frac{t + \log(M)}{\log(M)}} + \frac{C_1 (t + \log(M))}{\sqrt{n}}\right]
\right) \leq e^{-t}
\notag 
\end{align*}
for all $t\geq 0$.
Consequently the expectation of the $\max$-$\sup$ term can be bounded as 
\begin{align*}
&
\EE\left[\max_m 
\sup_{f_m \in \calH_m} 
\frac{|\frac{1}{n}\sum_{i=1}^n \sigma_i f_m(x_i)|}{\Uns(f_m)} \right] \notag \\
\leq 
&
K\left[C_s + 1 + \frac{C_1\log(M)}{\sqrt{n}}\right]
+ \int_{0}^{\infty} K\left[C_s + \sqrt{\frac{t + 1 + \log(M)}{\log(M)}} + \frac{C_1(t + 1 + \log(M))}{\sqrt{n}}\right]
e^{-t} \dd t \notag \\
\leq &
2K\left[C_s + 2 + \sqrt{\frac{\pi}{4 \log(M)}} + 2C_1 + \frac{C_1\log(M)}{\sqrt{n}}\right] \leq \tilde{C}_s, 
\notag 
\end{align*}
where we used $\sqrt{t + 1 + \log(M)} \leq \sqrt{t} + \sqrt{1 + \log(M)}$ and $\int_{0}^\infty \sqrt{t} e^{-t} \dd t = \sqrt{\frac{\pi}{4}}$, $\frac{\log(M)}{\sqrt{n}} \leq 1$, and 
$\tilde{C}_s = 2K[C_s + 2 +\sqrt{\frac{\pi}{4}}+ 3C_1]$.
\end{proof}

\begin{Lemma}
\label{lemm:uniformRatioBoundOnM}
If $\frac{\log(M)}{\sqrt{n}}\leq 1$, then under the Spectral Assumption (Assumption \ref{eq:specass}) and the Embedded Assumption (Assumption \ref{ass:linfbound}),  
the following holds 
\begin{align*}
&P\left(\max_m \sup_{f_m \in \calH_m} 
\frac{|\frac{1}{n}\sum_{i=1}^n \epsilon_i f_m(x_i)|}{\Uns(f_m)}
\geq K L \left[
2 \tilde{C}_s 
+
\sqrt{t} + \frac{C_1 t}{\sqrt{n}} \right]
\right) 
\leq
 e^{-t}. 
\end{align*}
\end{Lemma}
\begin{proof} [Proof of Lemma \ref{lemm:uniformRatioBoundOnM}]
By the contraction inequality \cite[Theorem 4.12]{Book:Ledoux+Talagrand:1991} and Lemma \ref{lemm:basicuniformlemm},
we have 
\begin{align*}
&
\EE\left[\max_m \sup_{f_m \in \calH_m} 
\frac{|\frac{1}{n}\sum_{i=1}^n \epsilon_i f_m(x_i)|}{\Uns(f_m)}\right] 
\leq 
2 \EE\left[\max_m \sup_{f_m \in \calH_m} 
\frac{|\frac{1}{n}\sum_{i=1}^n \sigma_i \epsilon_i f_m(x_i)|}{\Uns(f_m)}\right] \leq 
2 L \tilde{C}_s. 
\end{align*}
Using this and \Eqref{eq:ratioLtwonormBound} and \Eqref{eq:ratioinfnormBound},
Talgrand's inequality gives the assertion. 
\end{proof}

\begin{Theorem}
\label{th:SquareBound}
Let $\phi_s'=K[2C_1 \tilde{C}_s   +  C_1 + C_1^2]$. 
Then, 
if $\frac{\log(M)}{\sqrt{n}} \leq 1$,
we have for all $t\geq 0$
\begin{align*}
\left| \textstyle \left\|\sum_{m=1}^M f_m \right\|_n^2 - \left\|\sum_{m=1}^M f_m \right\|_{\LPi}^2 \right|  \leq 
\phi_s'
\sqrt{n} \left(
\sum_{m=1}^M \Uns(f_m) \right)^2  \eta(t),~~
\end{align*}
for all  $f_m \in \calH_m~(m=1,\dots,M)$
with probability $1- \exp( - t )$. 
\end{Theorem}
\begin{proof} [Proof of Theorem \ref{th:SquareBound}]
\begin{align}
&\EE\left[\sup_{f_m \in \calH_m} \frac{\left| \textstyle \left\|\sum_{m=1}^M f_m \right\|_n^2 - \left\|\sum_{m=1}^M f_m \right\|_{\LPi}^2 \right| }{
\left(\sum_{m=1}^M \Uns(f_m) \right)^2}  \right] \notag \\
\leq &  
2 \EE\left[ \sup_{f_m \in \calH_m} \frac{\textstyle  \left| \frac{1}{n} \sum_{i=1}^n \sigma_i ( \sum_{m=1}^M f_m(x_i))^2   \right| }{
\left(\sum_{m=1}^M \Uns(f_m) \right)^2}  \right] 
\notag \\
\leq &
\sup_{f_m \in \calH_m} 
\frac{\textstyle  \left\| \sum_{m=1}^M f_m \right\|_{\infty} }{
\sum_{m=1}^M \Uns(f_m) } \times
2 \EE\left[ 
\sup_{f_m \in \calH_m} 
\frac{\textstyle  \left| \frac{1}{n} \sum_{i=1}^n \sigma_i ( \sum_{m=1}^M f_m(x_i))   \right| }{
\sum_{m=1}^M  \Uns(f_m) }  \right],
\label{eq:squareUpperBoundtmp}
\end{align} 
where we used the contraction inequality in the last line \cite[Theorem 4.12]{Book:Ledoux+Talagrand:1991}.
Thus using \Eqref{eq:ratioinfnormBound}, the RHS of the inequality \eqref{eq:squareUpperBoundtmp} can be bounded as 
\begin{align*}
&2C_1 \sqrt{n} \EE\left[ \sup_{f_m \in \calH_m} 
\frac{\textstyle  \left| \frac{1}{n} \sum_{i=1}^n \sigma_i ( \sum_{m=1}^M f_m(x_i))   \right| }{
\sum_{m=1}^M \Uns(f_m) }  \right] \\
\leq & 
2 C_1 \sqrt{n} \EE\left[ \sup_{f_m \in \calH_m} \max_m 
\frac{\textstyle  \left| \frac{1}{n} \sum_{i=1}^n \sigma_i f_m(x_i)   \right| }{
\Uns(f_m)  }  \right],
\end{align*}
where we used the relation $\frac{\sum_{m} a_m}{\sum_m b_m} \leq \max_m(\frac{a_m}{b_m})$
for all $a_m \geq 0$ and $b_m \geq 0$ with a convention $\frac{0}{0}=0$.
By Lemma \ref{lemm:basicuniformlemm}, 
the right hand side is upper bounded by $2 C_1 \sqrt{n} \tilde{C}_s$.
Here we again apply Talagrand's concentration inequality, then we have
\begin{align}
&P\left(  \sup_{f_m \in \calH_m} \frac{ \left| \textstyle \left\|\sum_{m=1}^M f_m \right\|_n^2 - \left\|\sum_{m=1}^M f_m \right\|_{\LPi}^2 \right| }
{\left(\sum_{m=1}^M \Uns(f_m) \right)^2} 
\geq K\left[2C_1 \tilde{C}_s \sqrt{n}  + \sqrt{tn} C_1  + C_1^2 t   \right] \right) \leq e^{-t},
\end{align}
where we substituted the following upper bounds of $B$ and $U$.
\begin{align*}
B\leq &\sup_{f_m \in \calH_m} \EE \left[ \left( \frac{(\sum_{m=1}^M f_m)^2 }
{\left(\sum_{m=1}^M \Uns(f_m)\right)^2}\right)^2 \right] \\
\leq & 
\sup_{f_m \in \calH_m} \EE \left[ \frac{(\sum_{m=1}^M f_m)^2 } 
{\left( \sum_{m=1}^M  \Uns(f_m) \right)^2} 
\frac{(\|\sum_{m=1}^M f_m\|_{\infty})^2 } 
{ \left( \sum_{m=1}^M \Uns(f_m)\right)^2}  \right] \\
\leq & 
\sup_{f_m \in \calH_m}  \frac{\left(\sum_{m=1}^M \|f_m\|_{\LPi}\right)^2 } 
{\left( \sum_{m=1}^M  \Uns(f_m) \right)^2} 
\frac{(\sum_{m=1}^M C_1 \sqrt{n} \Uns(f_m))^2 } 
{ \left( \sum_{m=1}^M \Uns(f_m)\right)^2}  \\
\leq & \frac{C_1^2 n^2}{\log(M)} \leq C_1^2 n^2,
\end{align*}
where in the second inequality we used the relation $$\EE[(\sum_{m=1}^M f_m)^2] = \EE[ \sum_{m,m'=1}^M f_m f_{m'}] \leq \sum_{m,m'=1}^M \|f_m\|_{\LPi} \|f_{m'}\|_{\LPi} = (\sum_{m=1}^M \|f_m\|_{\LPi})^2$$
and in the third and forth inequality we used  \Eqref{eq:ratioinfnormBound} and 
\Eqref{eq:ratioLtwonormBound} respectively.
Here we again use \Eqref{eq:ratioLtwonormBound} to obtain 
\begin{align*}
U = &\sup_{f_m \in \calH_m} \left\| \frac{(\sum_{m=1}^M f_m)^2 }
{\left(\sum_{m=1}^M \Uns(f_m) \right)^2}\right\|_{\infty} 
\leq 
C_1^2 n.
\end{align*}
Therefore for $t \leftarrow \sqrt{n}t$, 
the above inequality implies 
the following inequality 
\begin{align}
&\sup_{f_m \in \calH_m} \frac{ \left| \textstyle \left\|\sum_{m=1}^M f_m \right\|_n^2 - \left\|\sum_{m=1}^M f_m \right\|_{\LPi}^2 \right| }
{\left(\sum_{m=1}^M \Uns(f_m) \right)^2}  
\leq 
K\left[2C_1 \tilde{C}_s  + C_1 + C_1^2  \right]\sqrt{n}\max(1,\sqrt{t},t/\sqrt{n}).
\end{align} 
with probability $1 - \exp( - t)$.
Remind $\phi_s' = K\left[2C_1 \tilde{C}_s   +  C_1 + C_1^2  \right]$, then we obtain the assertion. 
\end{proof}

Now we define 
\begin{align*}
\phi_s := \max \left( KL \left[2 \tilde{C}_s   +  1 + C_1  \right] ,K\left[2C_1 \tilde{C}_s   +  C_1 + C_1^2  \right] \right).
\end{align*}
We define events $\scrE_1(t)$ and $\scrE_2$ as 
\begin{align*}
&
\scrE_1(t) = \left\{ 
\left| \frac{1}{n}\sum_{i=1}^n \epsilon_i f_m(x_i)  \right|  \leq 
\phi_{s} 
\Uns(f_m)\eta(t) ,~\forall f_m \in \calH_m~(m=1,\dots,M)  \right\}, \\
&
\scrE_2(t') = \Bigg\{ 
\left| \textstyle \left\|\sum_{m=1}^M f_m \right\|_n^2 - \left\|\sum_{m=1}^M f_m \right\|_{\LPi}^2 \right| \leq 
\phi_{s}
\sqrt{n}
\left(\sum_{m=1}^M \Uns(f_m)\right)^2\eta(t'),\\
&~~~~~~~~~~~~~~~~\forall f_m \in \calH_m~(m=1,\dots,M) \Bigg\}.
\end{align*}
The following theorem immediately gives Theorem \ref{th:convergencerateofLpMKL}.
\begin{Theorem}
\label{th:convergencerateofLpMKLdet}
Let $\lambdatmp$ be an arbitrary positive real, and $\lambdaone$ satisfy $\lambdaone \geq \frac{1}{2}\lambdatmp$.
Then for all $n$ and $t'(>0)$ that satisfy  $\frac{\log(M)}{\sqrt{n}}\leq 1$ and 
$\phi_s \sqrt{n} \zeta_n^2 \eta(t')/\kminrho \leq \frac{1}{8}$,
we have 
\begin{align*}
& \|\fhat - \fstar\|_{\LPi}^2 
\leq 
\frac{ 8 }{ 3\kminrho} \eta(t)^2 \phi_s^2 \zeta_n^2
+ 
\frac{8}{3} \lambdaone \left( \sum_{m=1}^M \|\fstar_m \hnorm{m}^p \right)^{\frac{2}{p}},
\end{align*}
with probability $1- \exp(- t) - \exp(- t')$ for all $t \geq 1$.
\end{Theorem}
This gives Theorem \ref{th:convergencerateofLpMKL} by setting $\psi_s = 8\phi_s$. 

\begin{proof} [Proof of Theorem \ref{th:convergencerateofLpMKLdet}]
Since $y_i = \fstar(x_i) + \epsilon_i$, we have 
\begin{align*}
&\|\fhat - \fstar \|_{\LPi}^2 + \lambdaone \left( \sum_{m=1}^M \|\fhat_m\hnorm{m}^p \right)^{\frac{2}{p}} 
\notag \\
\leq 
&( \|\fhat - \fstar \|_{\LPi}^2 - \|\fhat - \fstar \|_{n}^2 ) + 
\frac{1}{n}\sum_{n=1}^n \sum_{m=1}^M\epsilon_i (\fhat_m(x_i) - \fstar_m(x_i)) + \lambdaone \left( \sum_{m=1}^M \|\fstar_m\hnorm{m}^p \right)^{\frac{2}{p}}.
\end{align*}
Here on the event $\scrE_2(t')$, the above inequality gives 
\begin{align}
&\|\fhat - \fstar \|_{\LPi}^2 + \lambdaone \left( \sum_{m=1}^M \|\fhat_m\hnorm{m}^p \right)^{\frac{2}{p}} 
\notag \\
\leq 
&  \phi_s \sqrt{n} \left(\sum_{m=1}^M \Uns(\fhat_m - \fstar_m) \right)^2 \eta(t') \! + \!
\frac{1}{n}\sum_{n=1}^n \sum_{m=1}^M\epsilon_i (\fhat_m(x_i) - \fstar_m(x_i)) + \lambdaone \left( \sum_{m=1}^M \|\fstar_m\hnorm{m}^p \right)^{\frac{2}{p}}.
\label{eq:basicineqLast}
\end{align}

Before we show the statements, we show two basic upper bounds of $\sum_{m=1}^M \Uns(f_m)$
required in the proof. 
First note that 
\begin{align*}
\left(\frac{\sum_{m=1}^M \|f_m\|_{\LPi}}{\sqrt{M}} + \frac{\lambdatmp^{\frac{1}{2}} \sum_{m=1}^M \|f_m\hnorm{m}}{M^{1-\frac{1}{p}}}\right) 
\leq &
\left(\sum_{m=1}^M \|f_m\|_{\LPi}^2\right)^{\frac{1}{2}}+ \lambdatmp^{\frac{1}{2}} \left(\sum_{m=1}^M \|f_m\hnorm{m}^p\right)^{\frac{1}{p}} 
\\
\leq &
2 
\left(\sum_{m=1}^M \|f_m\|_{\LPi}^2+ \lambdatmp \left(\sum_{m=1}^M \|f_m\hnorm{m}^p\right)^{\frac{2}{p}} \right)^{\frac{1}{2}}.
\end{align*}
Therefore we have 
\begin{align*}
\sum_{m=1}^M \Uns(f_m) \leq &
2 \left(\sqrt{\frac{M\log(M)}{n}}\vee \frac{\lambdatmp^{-\frac{s}{2}} M^{\frac{1-s}{2} + s(1-\frac{1}{p})}}{\sqrt{n}}
\vee 
\frac{M^{\frac{(1-s)^2}{2(1+s)} + (1-\frac{1}{p})\frac{s(3-s)}{1+s}} \lambdatmp^{-\frac{s(3-s)}{2(1+s)}}}{n^{\frac{1}{1+s}} }
\right) 
\times \notag \\
&\left(\sum_{m=1}^M \|f_m \|_{\LPi}^2 + \lambdatmp \left( \sum_{m=1}^M \|f_m \hnorm{m}^p \right)^{\frac{2}{p}}\right)^{\frac{1}{2}}.
\end{align*}
Reminding the definition of $\zeta_n$ (\Eqref{eq:defzetan}),
the above bound is equivalent to  
\begin{align}
\sum_{m=1}^M \Uns(f_m) \leq &
\zeta_n \left(\sum_{m=1}^M \|f_m \|_{\LPi}^2 + \lambdatmp \left( \sum_{m=1}^M \|f_m \hnorm{m}^p \right)^{\frac{2}{p}}\right)^{\frac{1}{2}}.
\label{eq:Unbounds}
\end{align}

\noindent{\it Step 1.}

By \Eqref{eq:Unbounds}, 
the first term on the RHS can be upper bounded as 
\begin{align*}
& \phi_s \sqrt{n} \left( \sum_{m=1}^M \Uns(\fhat_m - \fstar_m) \right)^2  \eta(t') 
\leq 
\phi_s \sqrt{n}\zeta_n^2 \eta(t') \left(\sum_{m=1}^M \|\fhat_m - \fstar_m \|_{\LPi}^2 + \lambdatmp \left( \sum_{m=1}^M \|\fhat_m - \fstar_m \hnorm{m}^p \right)^{\frac{2}{p}} \right)  \notag \\
\leq & \phi_s \sqrt{n}\zeta_n^2 \eta(t') \left(\frac{\|\fhat - \fstar \|_{\LPi}^2}{\kminrho} + 
\lambdatmp \left( \sum_{m=1}^M \|\fhat_m - \fstar_m \hnorm{m}^p \right)^{\frac{2}{p}} \right). 
\end{align*}
By assumption, we have $\phi_s \sqrt{n} \zeta_n^2 \eta(t')/\kminrho \leq \frac{1}{8}$.
Hence the RHS of the above inequality is bounded by 
\begin{align}
 \frac{1}{8} \left(\|\fhat - \fstar \|_{\LPi}^2 + \lambdatmp \left( \sum_{m=1}^M \|\fhat_m - \fstar_m \hnorm{m}^p \right)^{\frac{2}{p}} \right).
\label{eq:firsttermbound2}
\end{align}

~\\
\noindent{\it Step 2.}
On the event $\scrE_1(t)$, we have 
\begin{align}
&\phantom{\leq} \frac{1}{n}\sum_{n=1}^n \sum_{m=1}^M\epsilon_i (\fhat_m(x_i) - \fstar_m(x_i))  
\leq \sum_{m=1}^M \eta(t) \phi_s \Uns(\fhat_m - \fstar_m) \notag \\
& \leq \eta(t) \phi_s \zeta_n \left(\sum_{m=1}^M \|\fhat_m - \fstar_m \|_{\LPi}^2 + \lambdatmp \left( \sum_{m=1}^M \|\fhat_m - \fstar_m \hnorm{m}^p \right)^{\frac{2}{p}} \right)^{\frac{1}{2}} \notag \\
& \leq \frac{ 2 }{ \kminrho} \eta(t)^2 \phi_s^2 \zeta_n^2 + 
\frac{\kminrho}{8} \left( \sum_{m=1}^M \|\fhat_m - \fstar_m \|_{\LPi}^2 + \lambdatmp \left( \sum_{m=1}^M \|\fhat_m - \fstar_m \hnorm{m}^p \right)^{\frac{2}{p}} \right) \notag \\
& \leq \frac{ 2 }{ \kminrho} \eta(t)^2 \phi_s^2 \zeta_n^2 + 
\frac{1}{8} \left( \|\fhat - \fstar \|_{\LPi}^2 + \lambdatmp \left( \sum_{m=1}^M \|\fhat_m - \fstar_m \hnorm{m}^p \right)^{\frac{2}{p}} \right). \label{eq:secondtermbound} 
\end{align}

~\\
\noindent{\it Step 5.}

Substituting the inequalities 
\eqref{eq:firsttermbound2} and  \eqref{eq:secondtermbound} 
to \Eqref{eq:basicineqLast}, we obtain
\begin{align*}
&\frac{3}{4} \|\fhat - \fstar\|_{\LPi}^2 
+ \lambdaone \left( \sum_{m=1}^M \|\fhat_m \hnorm{m}^p \right)^{\frac{2}{p}} \\
\leq &
\frac{ 2 }{ \kminrho} \eta(t)^2 \phi_s^2 \zeta_n^2
+ 
\frac{\lambdatmp}{4}  \left( \sum_{m=1}^M \|\fhat_m - \fstar_m \hnorm{m}^p \right)^{\frac{2}{p}} + \lambdaone \left( \sum_{m=1}^M \|\fstar_m \hnorm{m}^p \right)^{\frac{2}{p}}.
\end{align*}
Now the second term of the RHS can be bounded as 
\begin{align*}
&\left( \sum_{m=1}^M \|\fhat_m - \fstar_m \hnorm{m}^p \right)^{\frac{2}{p}} 
\leq 
\left( \sum_{m=1}^M (\|\fhat_m \hnorm{m} + \| \fstar_m \hnorm{m})^p \right)^{\frac{2}{p}} 
\leq 
2 \left(  \sum_{m=1}^M \|\fhat_m \hnorm{m}^p \right)^{\frac{2}{p}} + 2 \left( \sum_{m=1}^M \| \fstar_m \hnorm{m}^p \right)^{\frac{2}{p}}. 
\end{align*}
Therefore we have 
\begin{align*}
&\frac{3}{4} \|\fhat - \fstar\|_{\LPi}^2 
\leq 
\frac{ 2 }{ \kminrho} \eta(t)^2 \phi_s^2 \zeta_n^2
+ 
2 \lambdaone \left( \sum_{m=1}^M \|\fstar_m \hnorm{m}^p \right)^{\frac{2}{p}}.
\end{align*}
This gives the assertion. 
\end{proof}

\section{Proof of Theorem \ref{th:LowerBounds} (minimax learning rate)}
\label{sec:proofOfMinimax}
\begin{proof} [Proof of Theorem \ref{th:LowerBounds}]
The proof utilizes the techniques developed by \citep{NIPS:Raskutti+Martin:2009,arXiv:Raskutti+Martin:2010} that applied the information theoretic technique developed by 
\citep{AS:Yang+Barron:99} to the MKL settings.
The $\delta$-packing number $Q(\delta,\calH,\LPi)$ of a function class $\calH$ is the largest number of functions $\{f_1, \dots, f_Q \} \subseteq \calH$
such that $\|f_i - f_j\|_{\LPi} \geq \delta$ for all $i\neq j$.
To simplify the notation, we write $\calF := \calHlp(R)$, $N(\varepsilon,\calH) := N(\varepsilon,\calH,\LPi)$ and $Q(\varepsilon,\calH) := Q(\varepsilon,\calH,\LPi)$.
It can be easily shown that $Q(2\varepsilon,\calF) \leq N(2\varepsilon,\calF) \leq Q(\varepsilon,\calF)$.

We utilize the following inequality given by Lemma 3 of \citet{NIPS:Raskutti+Martin:2009}:
\[
\min_{\hat{f}} \max_{f^* \in \calHlp(R_p)} \EE \|\hat{f} - f^* \|_{\LPi}^2 \geq 
\frac{\delta_n^2}{4}\left(1 - \frac{\log N(\varepsilon_n,\calF) + n \varepsilon_n^2/2 + \log 2}{\log Q(\delta_n,\calF)} \right).
\]

First we show the assertion for $p=\infty$. In this situation, there is a constant $C$ that depends only $s$ such that 
\begin{align*}
\log Q(\delta,\calF) \geq C M \log Q(\delta/\sqrt{M},\repH(R)),~~~
\log N(\varepsilon,\calF) \leq M \log N(\varepsilon/\sqrt{M},\repH(R)),
\end{align*}
(this is shown in Lemma 5 of \citet{arXiv:Raskutti+Martin:2010}, but we give the proof in Lemma \ref{lemm:QlogMbound} for completeness).
Using this expression, the minimax-learning rate is bounded as 
\[
\min_{\hat{f}} \max_{f^* \in \calHlp(R_p)} \EE \|\hat{f} - f^* \|_{\LPi}^2 \geq 
\frac{\delta_n^2}{4}\left(1 - \frac{ M \log N(\varepsilon_n/\sqrt{M},\repH(R)) + n \varepsilon_n^2/2 + \log 2}{M \log Q(\delta_n/\sqrt{M},\repH(R))} \right).
\]
Here we choose $\varepsilon_n$ and $\delta_n$ to satisfy the following relations: 
\begin{align}
&\frac{n}{2\sigma^2} \varepsilon_n^2 \leq M \log N\left(\varepsilon_n/\sqrt{M},\repH(R)\right), \label{eq:epssqMN} \\
&4 \log N\left(\delta_n/\sqrt{M},\repH(R)\right) \leq C \log Q\left(\delta_n/\sqrt{M},\repH(R)\right). \label{eq:NleqM}
\end{align}
With $\varepsilon_n$ and $\delta_n$ that satisfy the above relations \eqref{eq:epssqMN} and \eqref{eq:NleqM}, we have 
\begin{align}
\min_{\hat{f}} \max_{f^* \in \calHlp(R_p)} \EE \|\hat{f} - f^* \|_{\LPi}^2 \geq \frac{\delta_n^2}{16}.
\label{eq:basicMinimaxBound}
\end{align}
The relation \eqref{eq:epssqMN} can be rewritten as 
$$
\frac{n}{2\sigma^2} \varepsilon_n^2 \leq C M \left(\frac{\varepsilon_n}{R \sqrt{M}} \right)^{-2s}.
$$
It is sufficient to impose
$$
\varepsilon_n^2 \leq C n^{-\frac{1}{1+s}} M R^{\frac{2s}{1+s}},
$$
with a constant $C$.
The relation \eqref{eq:NleqM} can be satisfied by taking $\delta_n = c\varepsilon_n$ with an appropriately chosen constant $c$.
Thus \Eqref{eq:basicMinimaxBound} gives 
\begin{align}
\min_{\hat{f}} \max_{f^* \in \calHlp(R_p)} \EE \|\hat{f} - f^* \|_{\LPi}^2 \geq C n^{-\frac{1}{1+s}} M R^{\frac{2s}{1+s}},
\label{eq:MinimaxBoundInfty}
\end{align}
with a constant $C$. This gives the assertion for $p=\infty$.

Finally we show the assertion for $1\leq p < \infty$. Note that $\calHl{\infty}(R/M^{\frac{1}{p}}) \subset \calHlp(R)$ (this is because, 
for $\{x_m\}_{m=1}^M$ s.t. $|x_m|\leq R/M^{\frac{1}{p}}~(\forall m)$, we have $\sum_{m=1}^M |x_m|^p \leq M \left(R/M^{\frac{1}{p}}\right)^p = R^p$). Therefore we have 
\begin{align*}
\min_{\hat{f}} \max_{f^* \in \calHlp(R)} \EE \|\hat{f} - f^* \|_{\LPi}^2 & \geq \min_{\hat{f}} \max_{f^* \in \calHl{\infty}\left(R/M^{\frac{1}{p}}\right)} \EE \|\hat{f} - f^* \|_{\LPi}^2  \\
& \geq C n^{-\frac{1}{1+s}} M \left(R/M^{\frac{1}{p}} \right)^{\frac{2s}{1+s}}~~~~~~~~(\because \text{\Eqref{eq:MinimaxBoundInfty}})  \\
& \geq C n^{-\frac{1}{1+s}} M^{1-\frac{2s}{p(1+s)}} R^{\frac{2s}{1+s}}.
\end{align*}
This concludes the proof. 
\end{proof}

\begin{Lemma}
\label{lemm:QlogMbound}
There is a constant $C$ such that
$$
\log Q(\delta,\calHl{\infty}(R)) \geq C M \log Q(\delta/\sqrt{M},\repH(R)),
$$
for sufficiently small $\delta$.
\end{Lemma}
\begin{proof}
The proof is analogous to that of Lemma 5 in \citep{arXiv:Raskutti+Martin:2010}. 
We describe the outline of the proof. 
Let $N = Q(\sqrt{2} \delta/\sqrt{M},\repH(R))$ and $\{f_m^1,\dots,f_m^N\}$ be a $\sqrt{2}\delta/\sqrt{M}$-packing of $\calH_m(R)$.
Then we can construct a function class $\Upsilon$ as 
$$
\Upsilon = \left\{f^{\boldsymbol{j}} = \sum_{m=1}^M f_m^{j_m} \mid \boldsymbol{j} = (j_1,\dots,j_M) \in \{1,\dots,N\}^M \right\}.
$$

We denote by $[N] := \{1,\dots,M\}$.
For two functions $f^{\boldsymbol{j}},f^{\boldsymbol{j}'} \in \Upsilon$, 
we have by the construction
$$
\|f^{\boldsymbol{j}} - f^{\boldsymbol{j}'}\|_{\LPi}^2 = \sum_{m=1}^M \|f_m^{j_m} - f_m^{j'_m}\|_{\LPi}^2 \geq \frac{2\delta^2}{M} \sum_{m=1}^M \boldsymbol{1}[j_m \neq j'_m].
$$
Thus, it suffices to construct a sufficiently large subset $A \subset [N]^M$ such that all different pairs $\boldsymbol{j},\boldsymbol{j}'\in A$ 
have at least $M/2$ of Hamming distance $d_H(\boldsymbol{j},\boldsymbol{j}') := \sum_{m=1}^M \boldsymbol{1}[j_m \neq j'_m]$.

Now we define $d_H(A,\boldsymbol{j}) := \min_{\boldsymbol{j}' \in A} d_H(\boldsymbol{j}',\boldsymbol{j})$.
If $|A|$ satisfies 
\begin{align}
\label{eq:basicAbound}
\left|\left\{\boldsymbol{j} \in [N]^M ~\Big|~ d_H(A,\boldsymbol{j})\leq \frac{M}{2} \right\}\right|
< |[N]^M|=N^M,
\end{align}
then there exists a member $\boldsymbol{j}' \in [N]^M$ such that $\boldsymbol{j}'$ is more than $\frac{M}{2}$ away from $A$ with respect to $d_H$, i.e. $d_H(A,\boldsymbol{j}')>\frac{M}{2}$.
That is, we can add $\boldsymbol{j}'$ to $A$ as long as \Eqref{eq:basicAbound} holds.
Now since 
\begin{align}
\left|\left\{\boldsymbol{j} \in [N]^M ~\Big|~ d_H(A,\boldsymbol{j})\leq \frac{M}{2} \right\}\right|
\leq |A|  {M \choose M/2}N^{M/2}, 
\end{align}
\Eqref{eq:basicAbound} holds as long as $A$ satisfies
$$
|A| \leq \frac{1}{2} \frac{N^M}{{M \choose M/2}N^{M/2}} =: Q^*.
$$
The logarithm of $Q^*$ can be evaluated as follows 
\begin{align*}
\log Q^* & = \log\left(\frac{1}{2} \frac{N^M}{{M \choose M/2}N^{M/2}}\right)  = M\log N - \log 2 - \log{M \choose M/2} - \frac{M}{2} \log N \\
&\geq  \frac{M}{2} \log N - \log 2 - \log 2^M \geq \frac{M}{2} \log \frac{N}{16}.
\end{align*}
There exists a constant $C$ 
such that $N = Q(\sqrt{2} \delta/\sqrt{M},\repH(R)) \geq C Q(\delta/\sqrt{M},\repH(R))$ because  $Q(\delta,\repH(R)) \sim \left(\frac{\delta}{R} \right)^{-2s}$.
Thus we obtain the assertion for sufficiently large $N$.
\end{proof}

{ 
\bibliography{main} 
}

\end{document}